\documentclass[12pt]{l4dc2023}

\usepackage{graphics} 
\usepackage{mathptmx} 
\usepackage{bm}
\usepackage{booktabs}
\usepackage{wrapfig}
\usepackage{xcolor}
\usepackage{makecell}
\usepackage{caption}
\DeclareCaptionLabelFormat{andtable}{#1~#2  \&  \tablename~\thetable}
\DeclareCaptionLabelFormat{andfigure}{#1~#2  \&  \figurename~\thefigure}

\newcommand{\figref}{Fig.~\ref}
\newcommand{\tabref}{Tab.~\ref}
\newcommand{\secref}{Sec.~\ref}

\DeclareMathOperator*{\argmax}{arg\,max}

\title[Stability Attention]{Learning Stability Attention in Vision-based\\ End-to-end Driving Policies}
\usepackage{times}

\author{\Name{Tsun-Hsuan Wang}$^{\ast}$ \Email{tsunw@mit.edu}\\ 
\Name{Wei Xiao}\thanks{The authors are with equal contributions.}
\Email{weixy@mit.edu}\\ 
\Name{Makram Chahine}
\Email{chahine@mit.edu}\\ \Name{Alexander Amini}
\Email{amini@mit.edu} \\
\Name{Ramin Hasani}
\Email{rhasani@mit.edu}\\ \Name{Daniela Rus}  \Email{rus@mit.edu}\\
\addr Computer Science and Artificial Intelligence Lab, MIT, Cambridge, MA, USA.
}

\begin{document}

\maketitle

\begin{abstract}

Modern end-to-end learning systems can learn to explicitly infer control from perception. However, it is difficult to guarantee stability and robustness for these systems since they are often exposed to unstructured, high-dimensional, and complex observation spaces (e.g., autonomous driving from a stream of pixel inputs). 
We propose to leverage control Lyapunov functions (CLFs) to
equip end-to-end vision-based policies with stability properties and introduce stability attention in CLFs (att-CLFs) to tackle environmental changes and improve learning flexibility. We also present an uncertainty propagation technique that is tightly integrated into att-CLFs. 
We demonstrate the effectiveness of att-CLFs via comparison with classical CLFs, model predictive control, and vanilla end-to-end learning in a photo-realistic simulator and on a real full-scale autonomous vehicle.

\end{abstract}

\section{Introduction}
\label{sec:intro}



End-to-end representation learning systems are of effective methods to achieve autonomy in high-dimensional observation spaces such as vision where ground-truth states are unavailable (situations that are difficult to solve by planning and control algorithms \citep{primbs1999nonlinear}).   The overarching challenge for end-to-end systems is, however, guaranteeing their stability and robustness, which are necessary conditions for their deployment in the real world. The matters of stability and robustness of these systems are more critical for high relative degree systems such as autonomous driving. It is also more challenging to derive stability and robustness conditions for end-to-end systems as their observation space is large (e.g., streams of pixel inputs), resulting in increased uncertainty. For instance, in autonomous driving, even a slight change in the sunlight could drastically affect the output behavior of the autonomous control system.
Many recent works marry the expressive power of neural networks with classical control modules that empower learning-based models with stability properties \citep{chang2019neural,dai2021lyapunov} and safety guarantee \citep{Xiao2021bnet}. Advances in differentiable optimization further facilitate seamless integration of control theory to deep neural networks \citep{Amos2017,Amos2018}. While learning-based approaches have demonstrated promising potential in state-space control, there are limited attempts to extend this paradigm to end-to-end learning systems that infer control directly from perception.

In this paper, we propose to learn stability attention in control Lyapunov functions (CLFs) to equip end-to-end vision-based learning systems with stability properties. Lyapunov functions \citep{Sontag1983,Artstein1983,Freeman1996} are powerful tools in verifying the stability of a dynamical system. They have been vastly applied to control systems to formulate stability conditions via CLFs \citep{Aaron2012}. 
The time domain is discretized, and the state is assumed to be constant within each time step. The control problem then becomes a sequence of quadratic programs (QPs). 
In spite of this simple yet effective formulation, it is not straightforwardly ready for integration with end-to-end learning models due to (i) the incapability to respond to environmental changes based on perceptual inputs (ii) limited flexibility from exponentially stabilizing behavior.



To this end, we make CLFs dependent on the observation and incorporate them into end-to-end learning via differentiable QPs \citep{Amos2017}. Then, we show how to properly construct LF to achieve stability attention mechanism (att-CLFs) that can handle environmental changes and higher relative degree. 
\textbf{Contributions:} (i) we leverage CLFs to equip vision-based end-to-end driving policies with stability properties and introduce stability attention in CLFs to tackle environmental changes and improve learning flexibility (ii) we present an uncertainty propagation method tailored to differentiable optimization (iii) we compare our method with classical CLFs, model predictive control and direct end-to-end learning methods, and verify our method in a photo-realistic simulator and deploy our model on a real full-scale autonomous vehicle.

\begin{figure}[t]
	\centering
	\includegraphics[width=0.8\linewidth]{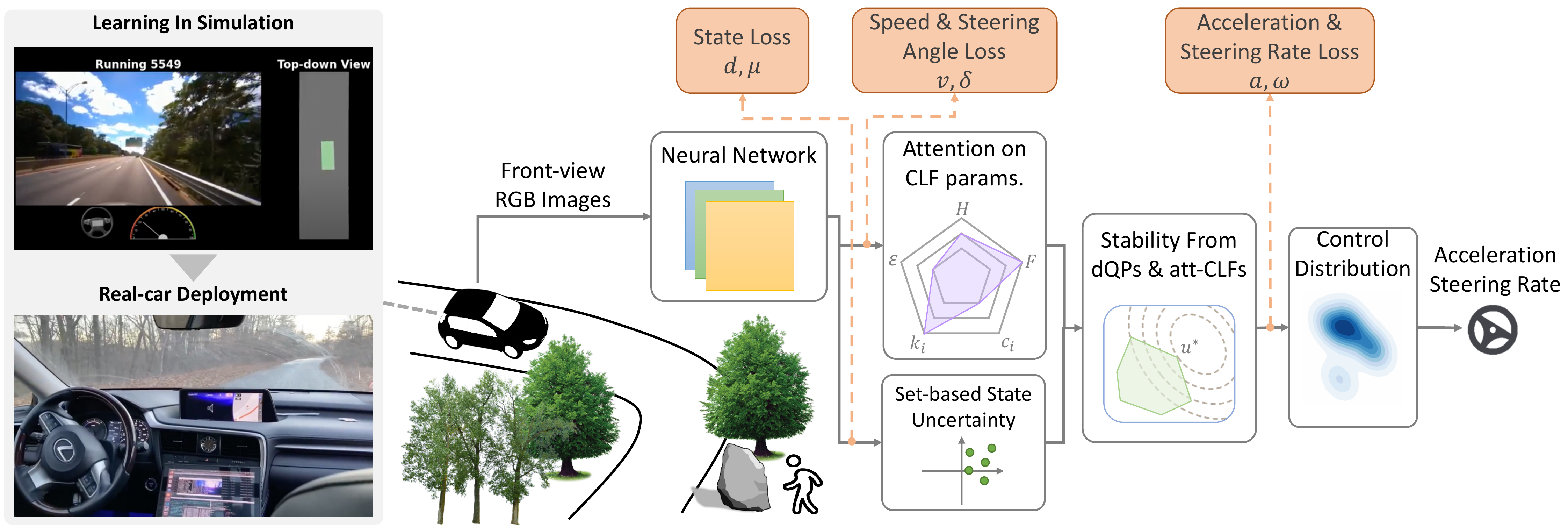}
	\caption{An overview of our end-to-end learning system with att-CLFs.}
	\label{fig:learn_frame}
	\vspace{-5mm}
\end{figure}







\section{Preliminaries}
\label{sec:background}
Consider an affine control system of the form:
\begin{equation}\label{eqn:affine}
\dot{\bm{x}}=f(\bm x)+g(\bm x)\bm u,
\end{equation}
where $\bm x\in\mathbb{R}^{n}$, $f:\mathbb{R}^{n}\rightarrow\mathbb{R}^{n}$
and $g:\mathbb{R}^{n}\rightarrow\mathbb{R}^{n\times q}$ are locally
Lipschitz, and $\bm u\in U\subset\mathbb{R}^{q}$, where $U$ denotes a control constraint set. The relative degree definition of a function is omitted, but can be found in \citep{Xiao2019}.

\begin{definition}  \label{def:clf}
 	({\it Control Lyapunov function} \citep{Aaron2012}) A continuously differentiable function $V: \mathbb{R}^n\rightarrow \mathbb{R}$ is a globally and exponentially stabilizing control Lyapunov function (CLF) for system (\ref{eqn:affine}) if there exist $c_1 >0, c_2>0, c_3>0$ such that for all $\bm x\in \mathbb{R}^n$, $c_1||\bm x||^2 \leq V(\bm x) \leq c_2 ||\bm x||^2$ and
 	\begin{equation}\label{eqn:clf}
 	\underset{u\in U}{inf} \lbrack L_fV(\bm x)+L_gV(\bm x) \bm u + c_3V(\bm x)\rbrack \leq 0.
 	\end{equation}
 	where $L_f, L_g$ denote the Lie derivatives \citep{Khalil2002} along $f$ and $g$, respectively.
 \end{definition}
 
 \begin{theorem} \label{thm:clf}
 \citep{Aaron2012} Given an exponentially stabilizing CLF $V$ as in Def. \ref{def:clf}, any Lipschitz continuous controller $ \bm u \in K_{clf}(\bm x)$, with
 $K_{clf}(\bm x) := \{\bm u\in U: L_fV(\bm x)+L_gV(\bm x) \bm u + c_3V(\bm x) \leq 0\},$
 exponentially stabilizes system (\ref{eqn:affine}) to the origin.
 \end{theorem}
 
\begin{figure}[t] 
	\centering
	\includegraphics[width=0.9\linewidth]{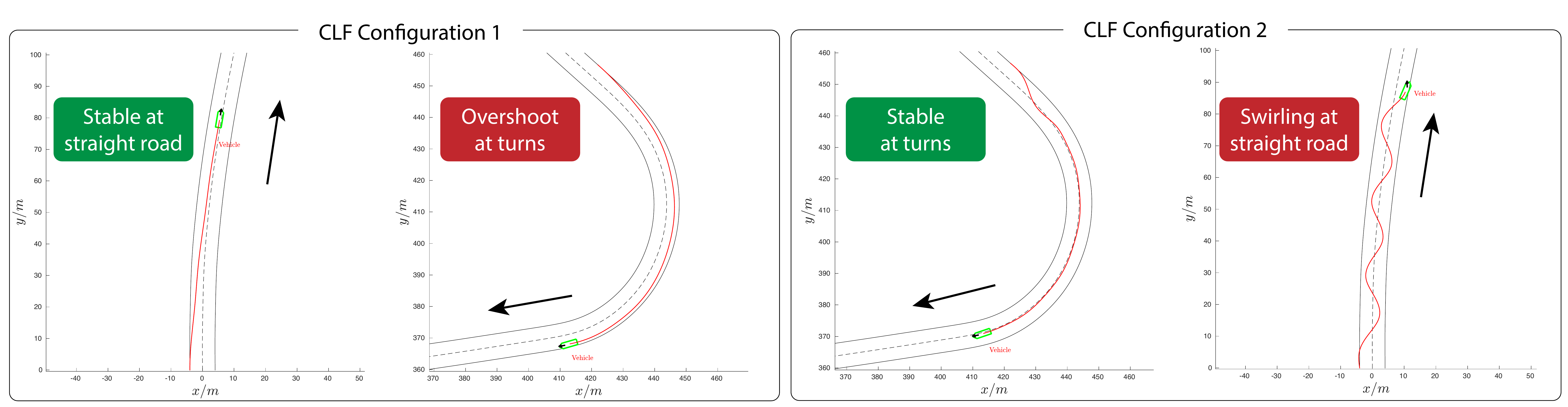}
	\caption{CLF-based QPs for lane tracking. A slow convergent CLF (left) may work well in a straight road, but fails to work in sharp turn. A fast convergent CLF (right) may work well in sharp turn, but fails to work in a straight road due to the myopic property of the CLF-based QP method.  
	}
	\label{fig:GsCLF}
	\vspace{-4mm}
\end{figure}

Consider an optimal control problem with the cost defined as $\frac{1}{2}\bm u(t)^TH\bm u(t) + F^T\bm u(t)$, where $H\in\mathbb{R}^{q\times q}$ is positive definite, $F\in\mathbb{R}^q$. Suppose the objective is to stabilize system (\ref{eqn:affine}) to the origin. We can formulate the following CLF-based optimization \citep{Aaron2012}: 
\begin{equation} \label{eqn:neuron0}
\begin{aligned}
\bm u^*(t) = \arg\min_{\bm u(t)\in U} \frac{1}{2}\bm u(t)^TH\bm u(t) + F^T\bm u(t), \\
\text{s.t.  }
 L_fV(\bm x) + L_gV(\bm x)\bm u + \epsilon V(\bm x) \leq 0, \end{aligned}
\end{equation}
where $\epsilon > 0$. The parameter $\epsilon$ in the CLF constraint and how we define $V(\bm x)$ significantly affect the system trajectory in converging to the origin. In order to solve (\ref{eqn:neuron0}), as in \citep{Aaron2012}, one usually discretizes the time, and the CLF constraint is linear in control since the state value is fixed at the beginning of the discrete interval. Therefore, each optimization is a quadratic program (QP). The optimal control obtained by solving each QP is held constant and applied at the current time interval to system (\ref{eqn:affine}). This process repeats until the final time. Note that the CLF constraint in (\ref{eqn:neuron0}) may conflict with the control bound. One usually relaxes the CLF constraint by replacing 0 in the right hand side with a slack variable, and also minimizes the slack variable in the cost function.

\noindent\textbf{Limitations of CLFs.} The aforesaid setup has limited flexibility in real world applications due to: 

\begin{itemize}
    \item The desired trajectories may not always exhibit exponentially stabilizing behavior.
    \item Under heavy environmental state changes, global stability via CLFs are almost impossible to obtain. Especially for complex control synthesis like vision-based driving, the definition of a CLF may depend on environmental variables.
    \item The CLF-based QP is solved point-wise yet can lead to aggressiveness and sub-optimalilty.
    \vspace{-3mm}
\end{itemize}

\begin{example} Suppose we wish to stabilize the vehicle to the lane center, the definition of a CLF would also involve  road curvature. When the road has sharp turns that are unexpected or not predictable, it is almost impossible to find a globally stabilizing CLF, as shown in Fig. \ref{fig:GsCLF}.
Regarding the last point, it means that the optimality of control synthesis is limited, and the system may overshoot around the equilibrium due to all possible perturbations (such as noisy dynamics, road curvature variation, etc.), as shown by the left case in Fig. \ref{fig:GsCLF}. We may formulate a receding horizon optimization to address this issue. However, the optimization would no longer be convex. 
\end{example}
In this paper, we aim to address these limitations with the proposed attention CLFs (att-CLFs).


\section{Problem Formulation and Approach}

\noindent\textbf{Problem setup.} Given (i) a sensor measurement $\bm z\in\mathbb{R}^d$ (e.g., front-view RGB images of the vehicle in this work), where $d\in\mathbb{N}$ is its dimension, (ii) vehicle dynamics (\ref{eqn:affine}), (iii) a nominal controller $h^\star(\bm x, \bm z)=\bm u^\star$ (such as a model predictive controller), and (iv) a learning-based end-to-end perception-to-control model $h(\bm z | \theta)=\bm u$, parametrized by $\theta$, our goal is to find optimal parameters:
\begin{equation}\label{eq:prob_form}
    \theta^\star = \underset{\theta}{\arg\min}~\mathbb{E}_{\bm z}[l(h^\star(\bm x, \bm z), h(\bm z|\theta))]
\end{equation}
where $l$ is a similarity measure (or a loss function). We also wish to equip the learning-based end-to-end controller $h$ with stability properties which we will describe in the following.

\noindent\textbf{Approach.} Our approach to solve the above problem is based on an end-to-end trainable framework that includes a differentiable optimization layer \citep{Amos2017} formulated by the proposed att-CLF based QP. The proposed att-CLF is observation dependent, and thus, it is adaptive to changing environments. The model takes front-view RGB images $\bm z$ as inputs whose features are extracted by CNN. In order to enforce temporal coherency and prevent the model from being susceptible to sudden input change (e.g., affected by sunlight), we employ LSTM after the CNN.
The outputs of the LSTM are used to construct the att-CLF, as well as to determine other optimization parameters in the att-CLF based QP. The outputs of the QP are the controls $\bm u$ of the vehicle. The model is trained using imitation learning where the ground truth control is obtained by a model predictive controller. We also propose to propagate state uncertainty to control to improve performance in real-world deployment. The overview is shown in \figref{fig:learn_frame}.

\section{Method}
\label{sec:method}
In this section, we introduce our contribution which is the notion of CLFs with stability attention (att-CLFs) and show how we can propagate uncertainty to improve end-to-end control. 

\begin{wrapfigure}[11]{R}{0.3\linewidth}
	\centering
	\includegraphics[width=0.6\linewidth]{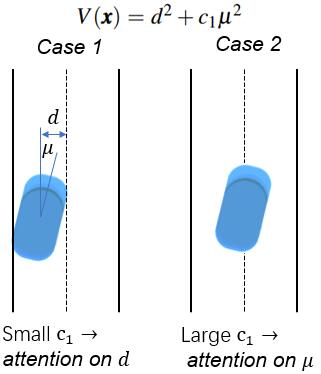}
	\caption{Lane tracking attention under different scenarios. }%
	\label{fig:attention0}%
\end{wrapfigure}

\subsection{CLFs with Stability Attention}
\label{sec:converge}
In this work, we only focus on quadratic CLFs due to its effectiveness and ubiquity \citep{Aaron2012}.
Specifically, we construct dependencies of att-CLFs on the environment, i.e., conditioning on sensor measurement $\bm z$:
\begin{equation}\label{eqn:lbf}
    V(\bm x, \bm z) = \bm x^TP(\bm z)\bm x
\end{equation}
where $P(\bm z)\in \mathbb{R}^{n\times n}$ is a positive definite matrix modelled by a neural network that takes $\bm z$ as inputs. However, learning to directly predict $P(\bm z)$ for valid and effective att-CLF is challenging. In the following, we describe how to properly construct $P(\bm z)$ for better learning. 
First, we decompose $P(\bm z) = Q^T(\bm z)\Lambda(\bm z) Q(\bm z)$, where $Q(\bm z)$ is an orthogonal matrix and $\Lambda(\bm z)$ is a diagonal matrix whose components are eigenvalues of $P(\bm z)$. We define a notion of stability attention in an att-CLF as:

\begin{definition}(Stability Attention) \label{def:attention}
The stability attention of an att-CLF $V(\bm x, \bm z)$ in (\ref{eqn:lbf}) under an observation $\bm z$ is defined by the state components of $Q(\bm z)\bm x$ (with change of basis) corresponding to the top $K\in\mathbb{N}$ largest eigenvalues of $P(\bm z)$ (i.e., the diagonal entries of $\Lambda(\bm z)$).
\end{definition}

\begin{example} \label{exam:track}
 For the autonomous driving example in Fig. \ref{fig:attention0}, we wish to stabilize the off-center distance $d \in\mathbb{R}$ and local heading error $\mu\in\mathbb{R}$ with respect to the lane center line to zero. An att-CLF can be defined as $V(\bm x,\bm z) = [d, \mu]^TP(\bm z)[d, \mu]$, where $P = [1, 0;0, c_1(\bm z)]$. Small $c_1(\bm z)$ denotes that we pay more attention to the off-center distance $d$, in which case a controller can quickly drive the vehicle position to the lane center if the vehicle is far from it. However, when the vehicle is near the lane center line, we wish to have a larger $c_1(\bm z)$, i.e., pay more attention to the local heading error $\mu$. This can make the vehicle orientation better aligned with the road direction.
\end{example}
\noindent In a CLF as in Def. \ref{def:clf}, the attention is fixed (constant $P$), leading to a fixed exponentially convergent behavior. In contrast, the attention of an att-CLF depends on $\bm z$, allowing flexibility to learn different types of behaviors and straightforward interpretation.
The dependency on observation makes att-CLFs become versatile and thus can imitate complex policies. 
When stabilizing full states, we can construct a diagonal $P(\bm z)$ (an att-CLF has relative degree one); otherwise, we proceed as follows.


\noindent\textbf{High relative degree.} In an att-CLF, it is possible that $\frac{\partial V(\bm x, \bm z)}{\partial {\bm x}}g(\bm x) = 0, \forall \bm x, \forall \bm z$, which prevents us from obtaining a stabilizing controller as in (\ref{eqn:neuron0}). This may happen when we only wish to stabilize partial states (like the driving problem in this work),
which introduces a high relative degree issue \citep{Xiao2019}. To address this, we propose state-feedback based att-CLFs, which can also be applied to non-input-to-state linearizable systems \citep{Khalil2002}. 
Suppose $\bm x := (x_1,\dots, x_n)$, and we only wish to stabilize $(x_1,\dots, x_m)$, where $m < n$. Suppose the relative degree of $x_i, i\in\{1,\dots,m\}$ is $r_i\in\mathbb{N}$, and define $n_0 = \sum_{i\in\{1,\dots,m\}}r_i$. We define the first derivative of $x_i$ as a new state $\xi_{i,1}\in\mathbb{R}$, and then define the derivative of $\xi_{i,1}$ as another new state $\xi_{i,2}\in\mathbb{R}$. This can be done recursively until we define a new state $\xi_{i,r_i-1}\in\mathbb{R}$ whose first derivative will be a linear function of $\bm u$ (i.e., its relative degree is one) for system (\ref{eqn:affine}) under a certain state $\bm x$.
 Then, in order to drive $x_i$ to 0, the desired state for $\xi_{i,r_i - 1}$ is $\hat \xi_{i,r_i - 1} := -l_i(\bm z)x_i - l_{i,1}(\bm z)\xi_{i,1}\dots -l_{i,r_i-2}(\bm z)\xi_{i,r_i-2}$, where $l_i(\bm z) > 0, l_{i,1}(\bm z)> 0,\dots, l_{i,r_i-2}(\bm z)> 0$. Intuitively, making $r_i-1$'th derivative being inversely proportional to lower degree states ($0$' to $r_i-2$'th) allows to drive eventually lower degree states to $0$.
 
Let $\bm y := (x_1,\dots,x_m, \xi_{1,1}, \dots, \xi_{1,r_1-1}, \dots, \xi_{m,1}, \dots, \xi_{m,r_m-1})$ be the new state vector of the transformed dynamics $\dot{\bm y} = \hat f(\bm y) + \hat g(\bm y)\bm u$, where $\hat f:\mathbb{R}^{n_0}\rightarrow \mathbb{R}^{n_0}, \hat g:\mathbb{R}^{n_0}\rightarrow \mathbb{R}^{n_0\times q}$ are determined by the above transformation process. We define a state-feedback based att-CLF with relative degree one, which stabilizes actual state $\xi_{i,r_i-1}$ to the desired state $\hat{\xi}_{i,r_i-1}$:
\begin{equation} \label{eqn:s-clf}
    V(\bm y, \bm z) = \sum_{i = 1}^mc_i(\bm z)(\frac{1}{l_i(\bm z)}\xi_{i,r_i-1} - \frac{1}{l_i(\bm z)}\hat \xi_{i,r_i-1})^2 =  \sum_{i = 1}^m c_i(\bm z)(x_i + k_{i,1}(\bm z)\xi_{i,1}\dots  + k_{i,r_i-1}(\bm z)\xi_{i,r_i-1})^2
\end{equation}
where $c_i(\bm z) > 0,\dots, c_m(\bm z)>0$, $k_{i,1}(\bm z) = \frac{l_{i,1}(\bm z)}{l_i(\bm z)}, \dots, k_{i,r_i-2}(\bm z) = \frac{l_{i,r_i-2}(\bm z)}{l_i(\bm z)}, k_{i,r_i-1}(\bm z) = \frac{1}{l_i(\bm z)}$. 
We can then construct $P(\bm z)$ with the new state vector $\bm y$ and rewrite the att-CLF in matrix form:
\begin{equation} 
    V(\bm y, \bm z) = \bm y^T P(\bm z)\bm y = (Q(\bm z)\bm y)^T\Lambda(\bm z)(Q(\bm z)\bm y),~~~~~\text{where}
\end{equation}
$$\small
Q(\bm z) = \left[
\begin{array}[c]{cc}%
    I_{m\times m} & k_{block}(\bm z) \\
    0_{(n_0-m)\times m} &  0_{(n_0-m)\times(n_0-m)}
\end{array}
\right], 
\Lambda(\bm z) = \left[
\begin{array}[c]{cc}%
   
    c_{m\times m} & 0_{m\times (n_0-m)}\\
 0_{(n_0-m)\times m}&
 0_{(n_0-m)\times(n_0-m)}
\end{array}
\right].
$$
where $k_{block}(\bm z)$ is a block diagonal matrix composed by vectors $k_i(\bm z) = (k_{i,1}(\bm z),\dots, k_{i,r_i - 1}(\bm z)), i\in\{1\dots, m\}$, $c_{m\times m}(\bm z)$ is a diagonal matrix composed by $c_i(\bm z), i\in \{1,\dots,m\}$, $0_{i\times j}$ denotes a zero matrix with dimension $i\times j$, and $I_{m\times m}$ denotes an identity matrix with dimension $m\times m$. Note that $k_i(\bm z),c_i(\bm z)$ are the output of the previous layer (LSTM).

As per Def. \ref{def:attention}, the stability attention in (8) is the top $K$ largest diagonal entries of $\Lambda(\bm z)$, and the corresponding state is $Q(\bm z)\bm y$ after change of basis. We can get a similar att-CLF constraint as in (\ref{eqn:clf}):
{\begin{equation}\label{eqn:dclf}
    \frac{\partial V(\bm y, \bm z)}{\partial {\bm y}}\hat f(\bm y) + \frac{\partial V(\bm y, \bm z)}{\partial {\bm y}}\hat g(\bm y)\bm u + \frac{\partial V(\bm y, \bm z)}{\partial \bm z}\dot{\bm z} + \epsilon V(\bm y, \bm z) \leq 0,
\end{equation}
}In the above, the term $\frac{\partial V(\bm y, \bm z)}{\partial \bm z}\dot{\bm z}$ can be evaluated using the bound of $\dot {\bm z}$ if it is available to ensure robustness to the input $\bm z$. We have the following theorem to prove the stability:.
\begin{theorem} \label{thm:att-clf}
Given a state-feedback based att-CLF as in (\ref{eqn:s-clf}), any control $\bm u$ that satisfies the att-CLF constraint (\ref{eqn:dclf}) stabilizes the partial state $(x_1,\dots, x_m)$ to the origin for system (\ref{eqn:affine}).
\end{theorem}
\begin{proof}
 The att-CLF constraint (\ref{eqn:dclf}) can be rewritten in the time derivative form $\frac{dV(\bm y,\bm z)}{dt} + \epsilon V(\bm y,\bm z) \leq 0$. Then, according to Lyapunov stability theorem \citep{Khalil2002}, we have that $V(\bm y,\bm z)$ will be exponentially stabilized to zero. In other words, when $V(\bm y,\bm z) = 0$, we have either case $(i)$ $x_i = 0, \xi_{i,j} = 0, \forall j\in\{1,\dots, r_i-1\}, \forall i\in\{1,\dots,m\}$ or case $(ii)$ $x_i = - k_{i,1}(\bm z)\xi_{i,1}\dots -k_{i,r_i-1}(\bm z)\xi_{i,r_i-1}$. In case $(i)$, the theorem holds. In case $(ii)$, following (\ref{eqn:s-clf}), the condition is equivalent to $\xi_{i,r_i - 1} = -l_i(\bm z)x_i - l_{i,1}(\bm z)\xi_{i,1}\dots -l_{i,r_i-2}(\bm z)\xi_{i,r_i-2}$.
Since we have $\dot x_i = \xi_{i,1}, \dots, \dot \xi_{i,r_i - 2} = \xi_{i,r_i - 1}$ following the above state transformation process, $\xi_{i,r_i - 1} = -l_i(\bm z)x_i - l_{i,1}(\bm z)\xi_{i,1}\dots -l_{i,r_i-2}(\bm z)\xi_{i,r_i-2}$ is the state feedback control law that drives $x_i,\forall i\in\{1,\dots,m\}$ to the origin \citep{Khalil2002}. Thus, any control $\bm u$ that satisfies the constraint (\ref{eqn:dclf}) stabilizes the partial state $(x_1,\dots, x_m)$ to the origin for system (\ref{eqn:affine}). 
\end{proof}

\noindent\textbf{An optimization layer in NN.} Then, we can formulate an att-CLF based QP as in (\ref{eqn:neuron0}) by replacing its CLF constraint with (\ref{eqn:dclf}). 
This optimization becomes a sequence of QPs using the time discretization from \citep{Aaron2012}, and the QP can be incorporated into a neural network as a differentiable optimization layer \citep{Amos2017} (\figref{fig:learn_frame}). This optimization layer allows error backpropagation through the solution of the QP and updates the trainable parameter utilizing gradient descent. 
The CLF constraint is with stability attention. 
Recall that we can get a relaxed form to handle the infeasibility of CLFs as mentioned in introduced in \citep{Aaron2012}. We may also add a differentiable Control Barrier Functions \citep{Xiao2021bnet} for obstacle avoidance.

\subsection{Uncertainty Propagation}
\label{ssec:uncertainty}

The accuracy of state $\bm{x}$ plays a critical role in CLFs. An unbiased estimation but with high variance can lead to undesired activeness of CLF constraint in the QP, producing overall correct but jerky control. We resort to mitigating this issue with the introduction of state uncertainties. We use a set-based representation for state uncertainty since it is highly flexible to represent any kind of distribution, can be constructed based on parametric models, and is widely adopted in predictive uncertainty quantification \citep{gal2016dropout,lakshminarayanan2017simple}. Given a set of state samples $\{\bm{x}_i\}_{i=1}^N$, we solve for a set of controls $\{\bm{u}^*_i\}_{i=1}^N=\text{QP}(\{\bm{x}_i\}_{i=1}^N)$, where $\text{QP}$ is defined in \eqref{eqn:neuron0} with \eqref{eqn:dclf} replacing the original CLF constraint and can be efficiently solved by batch operation. We then use kernel density estimation to convert the control sample set to a distribution,
\begin{equation}\label{eqn:kde}
    p(\bm{u}|\bm{x},\bm{z}) = \frac{1}{Nh}\sum_{i=1}^N K(\frac{\bm{u} - \bm{u}^*_i}{h}),
\end{equation}
\noindent where $K$ is a kernel function and $h>0$ is the bandwidth. We use the Gaussian kernel with Scott's rule \citep{scott2015multivariate} for bandwidth selection. With control distribution at hand, we can easily integrate other (unnormalized) prior to obtaining the final control with a set of probing points,
\begin{equation}\label{eqn:ctrl_dist}
    \bm{u}^* = \underset{\bm{u}\in \mathcal{U}}{\argmax}~p(\bm{u})p(\bm{u}|\bm{x},\bm{z}),
\end{equation}
\noindent where the probing point set $\mathcal{U}$ can be a grid that covers the control space or even simply the control samples obtained from the QP. The prior $p(\bm{u})$ can be any prior knowledge of the design or control distribution from other modules in the system. For simplicity, we use uniform distribution as prior. To obtain predictive uncertainty, we follow \citep{kendall2017uncertainties} that predicts the mean and log variance of a normal distribution for state estimate, from which we draw samples and construct the set $\{\bm{x}_i\}_{i=1}^N$. Note that we can flexibly plug in other predictive uncertainty quantification techniques.

\section{Experiments}
\label{sec:experiments}

\subsection{Platform \& Dataset}
\noindent\textbf{Hardware Setup.} We collect data and deploy our models on a full-scale vehicle (2019 Lexus RX 450H) with autonomous driving capability. The onboard computation includes an NVIDIA 2080Ti GPU and an AMD Ryzen 7 3800X 8-Core Processor. The primary perception sensor is a 30Hz BFS-PGE-23S3C-CS RGB camera with resolution $960\times 600$ and $130^\circ$ horizontal field-of-view. The car is also equipped with IMUs and wheel encoders that provide steering feedback and odometry, as well as a centimeter-level accurate OxTS differential GPS for evaluation purposes only.
\begin{table}[t]
    \centering
    \begin{minipage}[c]{0.55\linewidth}
        \centering
        \footnotesize
        \begin{tabular}[b]{cccc}
            \toprule
            \textbf{Method} & \textbf{Obs.} & \textbf{\makecell{Mean\\ Dev.}} & \textbf{\makecell{Infer.\\Time (s)}} \\ 
            \midrule
            NMPC & State & \bf{0.011} &  0.818 \\
            CLF \citep{Aaron2012} & State & 0.150 &
            0.005 \\
            att-CLF (ours) & State & 0.085 & 0.028 \\
            \midrule
            V-E2E \citep{amini2021vista} & Image & 0.573 & 0.002 \\
            E2E-CLF \citep{Amos2017} & Image & 0.328 & 0.033 \\
            att-CLF (ours) & Image & \bf{0.218} & 0.032 \\
            \bottomrule
        \end{tabular}
    \end{minipage}
    \hfill
    \begin{minipage}[c]{0.4\linewidth}
        \centering
        \includegraphics[trim=0 10mm 0 0,width=1\linewidth]{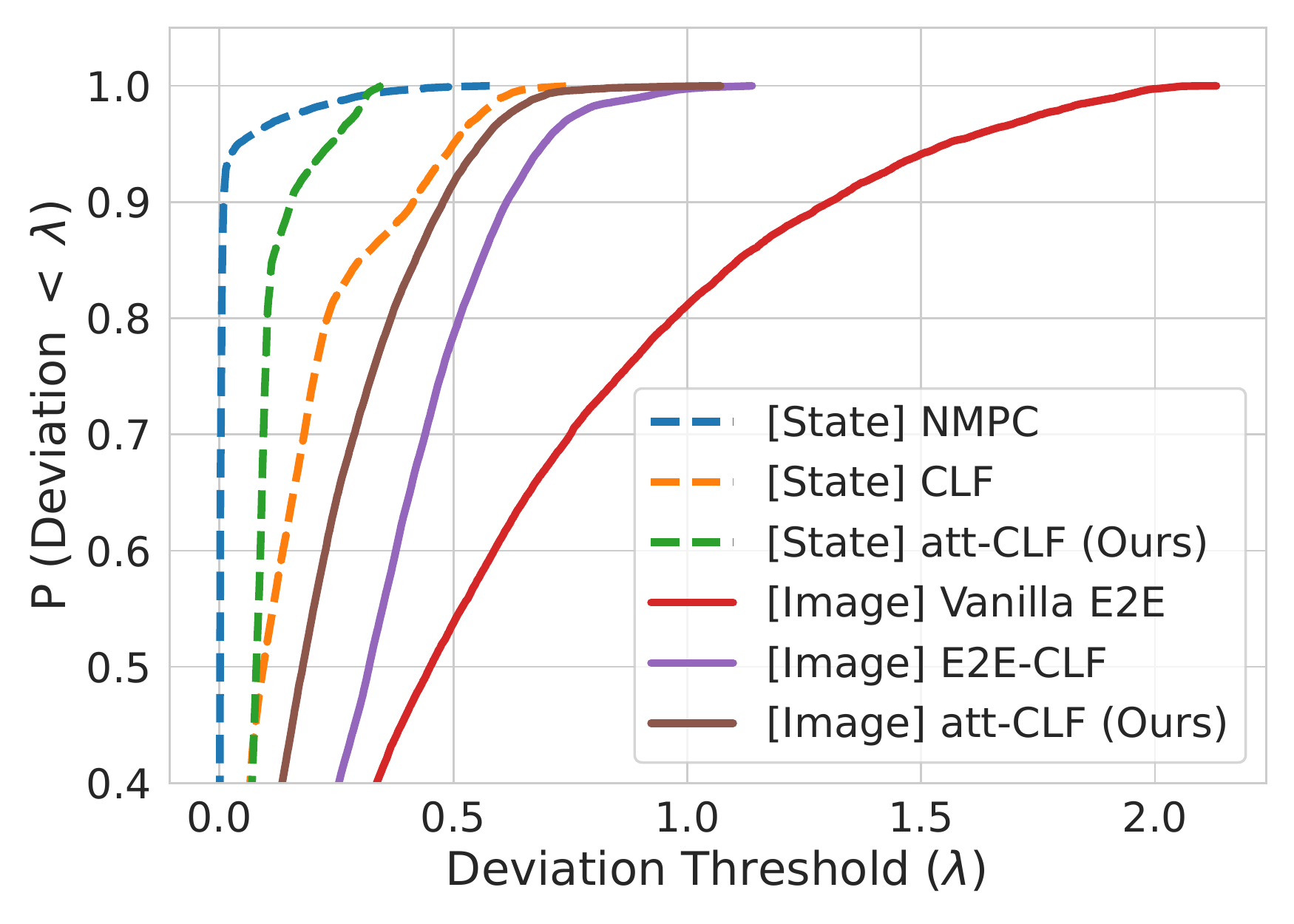}
    \end{minipage}
    \captionlistentry[figure]{A table beside a figure}
    \captionsetup{labelformat=andfigure}
    \label{fig:deviation_recall_plot_in_vista}
    \caption{\textbf{Left:} Quantitative comparison. \textbf{Right:} Probability of deviation from lane center smaller than a threshold $\lambda$ (the larger area under the curver, the better).}
    \label{tab:quant_res}
    \vspace{-5mm}
\end{table}

\noindent\textbf{Simulation and Data Generation.} We adopt a training strategy called guided policy learning that leverages a data-driven simulator \textit{VISTA} \citep{amini2021vista} to augment a real-world dataset with diverse synthetic data for robust policy learning. The ground truth controls (labels) are generated using nonlinear model predictive control (NMPC). Although computationally expensive, NMPC is tractable offline. 
The simulation is built upon roughly 2-hour real-world driving data collected at a different time of day, weather conditions, and seasons of a year, in total roughly 200k images.

\subsection{Implementation Details}
\label{ssec:imp_details}
\noindent{\bf Vehicle dynamics.} We consider the vehicle dynamics defined with respect to a reference trajectory as shown in \citep{Rucco2015}:
\begin{equation}\label{eqn:vehicle}
\dot s = \frac{v\cos(\mu + \beta)}{1 - d\kappa},\;\;\; \dot d =  v\sin(\mu + \beta), \;\;\;\dot \mu = \frac{v}{l_r}\sin\beta - \kappa\frac{v\cos(\mu + \beta)}{1 - d\kappa}, \;\;\;\dot v = u_a, \;\;\;\dot \delta = u_{\omega},
\end{equation}
where $\bm x = (s, d, \mu, v, \delta)$, $s$ and $d$ are respectively the along-trajectory distance and lateral distance of the vehicle Center of Gravity (CoG) with respect to the closest point on the reference trajectory; $\mu$ is local heading error determined by the global vehicle heading $\psi$ and the tangent angle $\phi$ of the closest point on the reference trajectory (i.e., $\psi = \phi + \mu$); 
$v$, $u_a$ denote the vehicle linear speed and acceleration; $\delta$, $u_\omega$ denote the steering angle and steering rate; $\kappa$ is the curvature of the reference trajectory at the closest point; $l_r$ is the length of the vehicle from the tail to the CoG and $\beta = \arctan\left(\frac{l_r}{l_r + l_f}\tan\delta\right)$, where $l_f$ is the length of the vehicle from the head to the CoG.
\begin{figure}[t]
    \centering
    \subfigure{\includegraphics[width=0.3\linewidth]{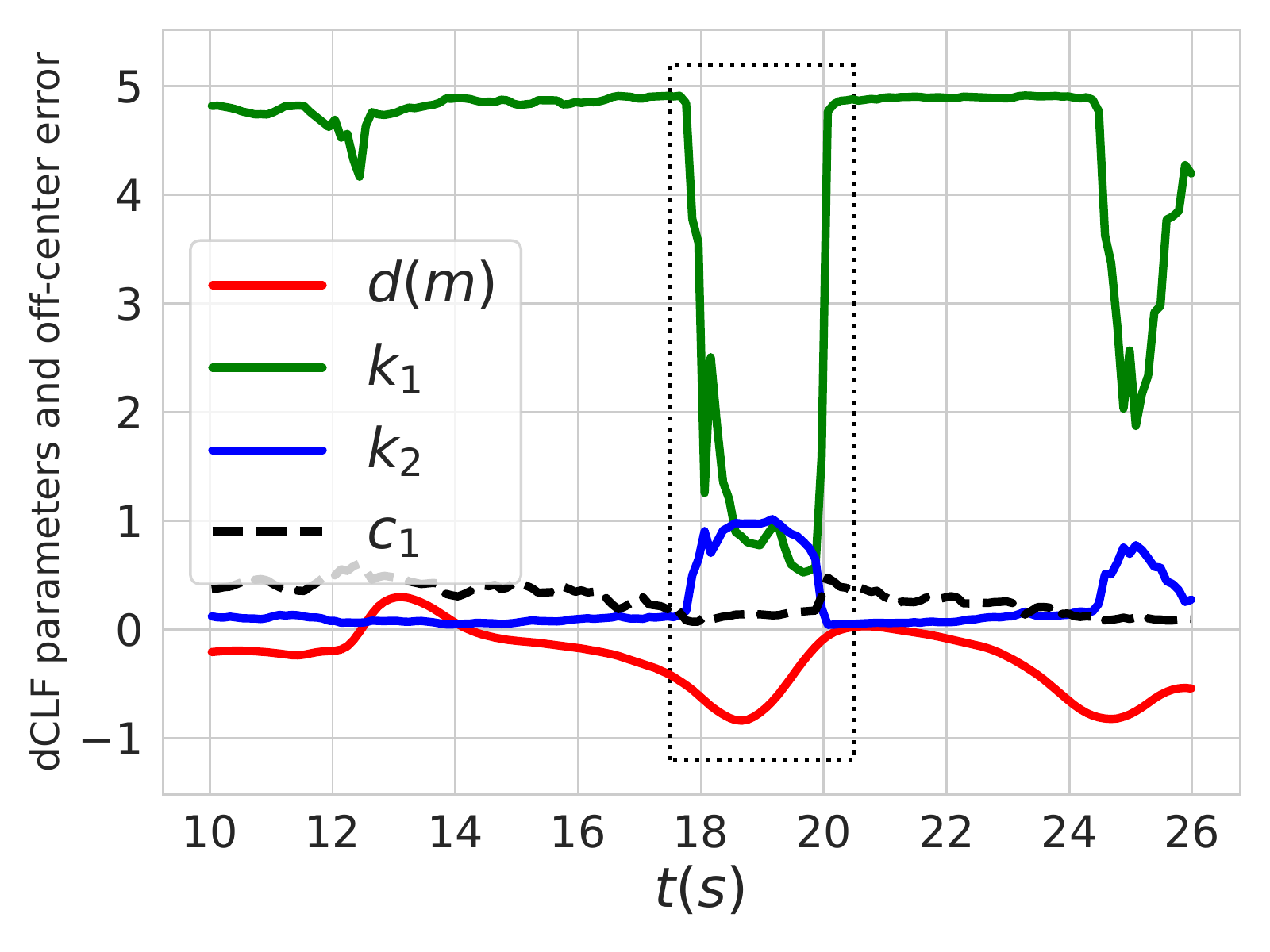}}
    \qquad
    \subfigure{\includegraphics[width=0.45\linewidth]{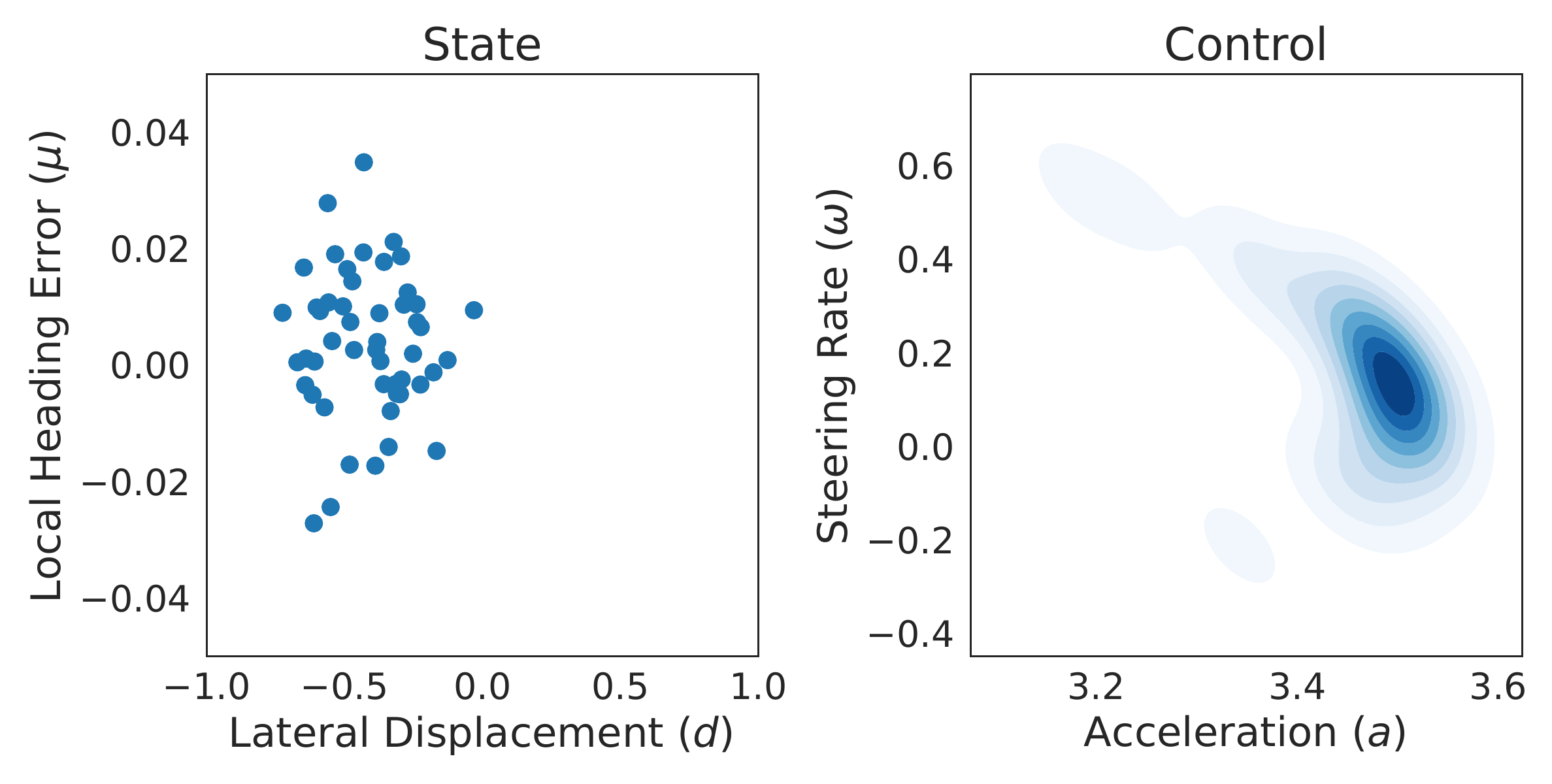}}
    \vspace{-3mm}
    \caption{\textbf{(a)} The profile of att-CLF attention $k_1,k_2,c_1$ under varying lateral off-center distance $d$. \textbf{(b)} Uncertainty propagation from sample set in state space to density in control space. 
    }
    \label{fig:uncertainty_prop}
    \vspace{-5mm}
\end{figure}

\noindent\textbf{Att-CLFs for end-to-end driving.} In this problem, we only wish to stabilize $d$ and $\mu$ to 0 for lane tracking. The relative degrees of both $d$ and $\mu$ are two. Thus, we formulate a state-feedback based att-CLF following Sec. \ref{sec:converge}. The desired states for $d$ and $\mu$ are $\hat d = - \xi_1, \hat \mu = - \xi_2$, where $\xi_1 = v\sin(\mu + \beta), \xi_2 = (\frac{v}{l_r}\sin\beta- \kappa\frac{v\cos(\mu + \beta)}{1 - d\kappa})$. The transformed state vector is then $\bm y = (d, \mu, \xi_1, \xi_2)$. The att-CLF in this case is then $V(\bm y, \bm z) = (d + k_1(\bm z) \xi_1)^2 + c_1(\bm z)\left(\mu + k_2(\bm z)\xi_2\right)^2$ (we set $c_2(\bm z)$ to one). We assume the observation is piece-wise constant $\dot{\bm z} = 0$, similar to control approximation in \citep{Aaron2012}. We leave the estimation of $\dot{\bm z}$ or its bound to future work.

\noindent\textbf{Model \& Learning.} The input is an RGB image. We extract features with a 5-layer CNN with kernel size as 3, channels as (24, 36, 48, 64, 64) and group norm, followed by average pooling. The feature vector is then fed to a 64-dim LSTM for temporal reasoning. We then use independent MLP heads of hidden size as (32, 32) for different outputs $u_a,u_\omega$ and intermediate estimates $v,\delta,d,\mu,\kappa,k_1,k_2,c_1$. 
We use L2 loss for control $u_a,u_\omega$ with ground truth obtained from NMPC. For state estimation $d,\mu,\kappa$, which is modeled as distribution as mentioned in \secref{ssec:uncertainty}, we follow predictive uncertainty loss in \citep{kendall2017uncertainties}. We also adopt auxiliary L2 loss on $v,\delta$ for more efficient and stable learning. We use Adam optimizer with a learning rate of 0.001. 

\subsection{Offline Closed-loop Evaluation}


We evaluate our policies in a closed-loop setting in \textit{VISTA} \citep{amini2021vista} based on a test set, where we collect 100 episodes (each of maximal 300 steps) with random scenes and initial poses. The baselines include i) \textit{nonlinear MPC} (NMPC) ii) classical \textit{CLF} \citep{Aaron2012} iii) \textit{vanilla end-to-end learning} (V-E2E) \citep{amini2021vista} and iv) \textit{end-to-end CLF} (E2E-CLF) where we use CNN+LSTM to extract state followed by differentiable QP (dQP) \citep{Amos2017} with fixed CLF parameters. 
The proposed att-CLF method (a) comparing to \textit{V-E2E} -- augment end-to-end learning with better lane tracking error (distribution)  (b) comparing to \textit{CLF} and \textit{E2E-CLF} -- outperform other CLF baselines with stability attention and (c) comparing to \textit{NMPC} -- runs much faster with simple dQP formulation, as shown in \figref{fig:deviation_recall_plot_in_vista} and \tabref{tab:quant_res}.


\noindent\textbf{Attention of stability.} 
When the car is more off-lane ($|d|$ is large), the att-CLF pays more attention to the lateral distance $d$ (larger $k_1$); conversely, when the car is closer to the lane center ($|d|$ is smaller), the att-CLF pays more attention to heading $\mu$ (larger $k_2$ and smaller $c_1$) to achieve stable lane following, as shown in \figref{fig:uncertainty_prop}a. Otherwise, the vehicle can overshoot around the lane center line due to the myopic solving method and state uncertainties. 
This result demonstrates that att-CLF can learn effective strategy by adapting its parameters in response to different situations.

\noindent\textbf{Uncertainty propagation.} 
A simple interpretation for \figref{fig:uncertainty_prop}b is to look at how the current state affects steering rate control $\omega$. While all state samples indicate the car is on the right of the lane center ($d<0$), there is some disagreement on whether the car heading is clockwise ($\mu<0$) or counter-clockwise ($\mu >0$)  with respect to the road curvature. In the case where $d<0,\mu<0.01$ (car is at the right to the lane center and does not orient to the extreme left), the att-CLF will most likely suggest turning left (steering rate $\omega > 0$), which leads to the major mode in the control distribution. 
Also, the set-based uncertainty can be computed via batch operation in dQP and has acceptable computational overhead  (28ms for 1 sample, 43ms for 10 samples, 50ms for 50 samples).


\begin{figure}
    \begin{minipage}[c]{0.5\linewidth}
    	\flushleft
    	\includegraphics[width=1\linewidth]{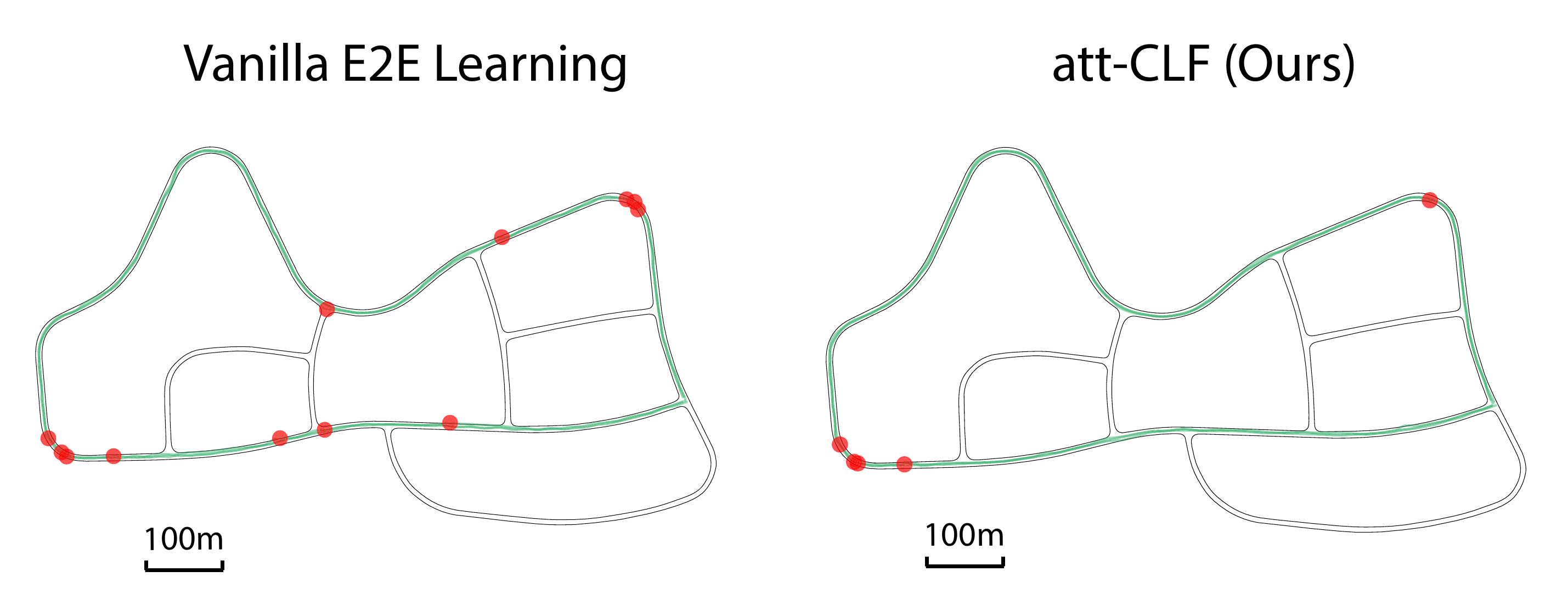} 
    \end{minipage}
    \begin{minipage}[c]{0.4\linewidth}
        \begin{center}
            \footnotesize
            \begin{tabular}{ccc}
            \toprule
            \textbf{Method} & \textbf{\makecell{Avg. \# of\\ Intervention}} & \textbf{\makecell{Mean\\ Deviation (m)}} \\ 
            \midrule
            Vanilla E2E & 4.0 $\pm$ 1.6 & 0.779 $\pm$ 0.178 \\
            att-CLF (ours) & \textbf{1.6 $\pm$ 0.4} & \textbf{0.687 $\pm$ 0.083} \\
            \bottomrule
            \end{tabular}
        \end{center}
    \end{minipage}
    \vspace{-3mm}
    \captionlistentry[table]{}
    \captionsetup{labelformat=andtable}
    \label{tab:real_car_quant}
    \caption{\textbf{Left:} Real-world deployment. Policy trajectories (3 trials at different times of day) and intervention locations (red dot). \textbf{Right:} Average performance of the real-car test.}
    \label{fig:devens_loop}
    \vspace{-5mm}
\end{figure}

\subsection{Real-car Deployment}

\begin{figure}[b]
\vspace{-4mm}
	\centering
    \subfigure[]{
	\includegraphics[width=0.55\linewidth]{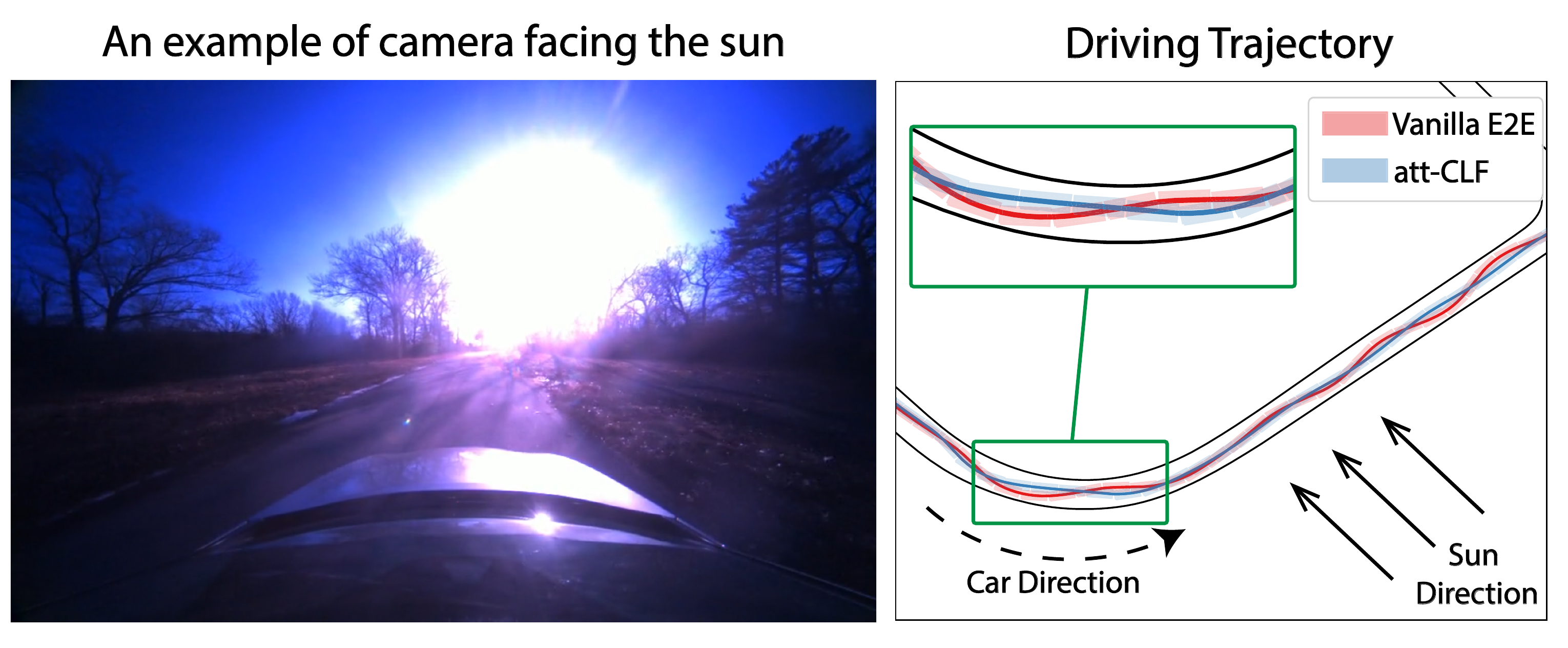}
	}
	\subfigure[]{\includegraphics[width=0.3\linewidth]{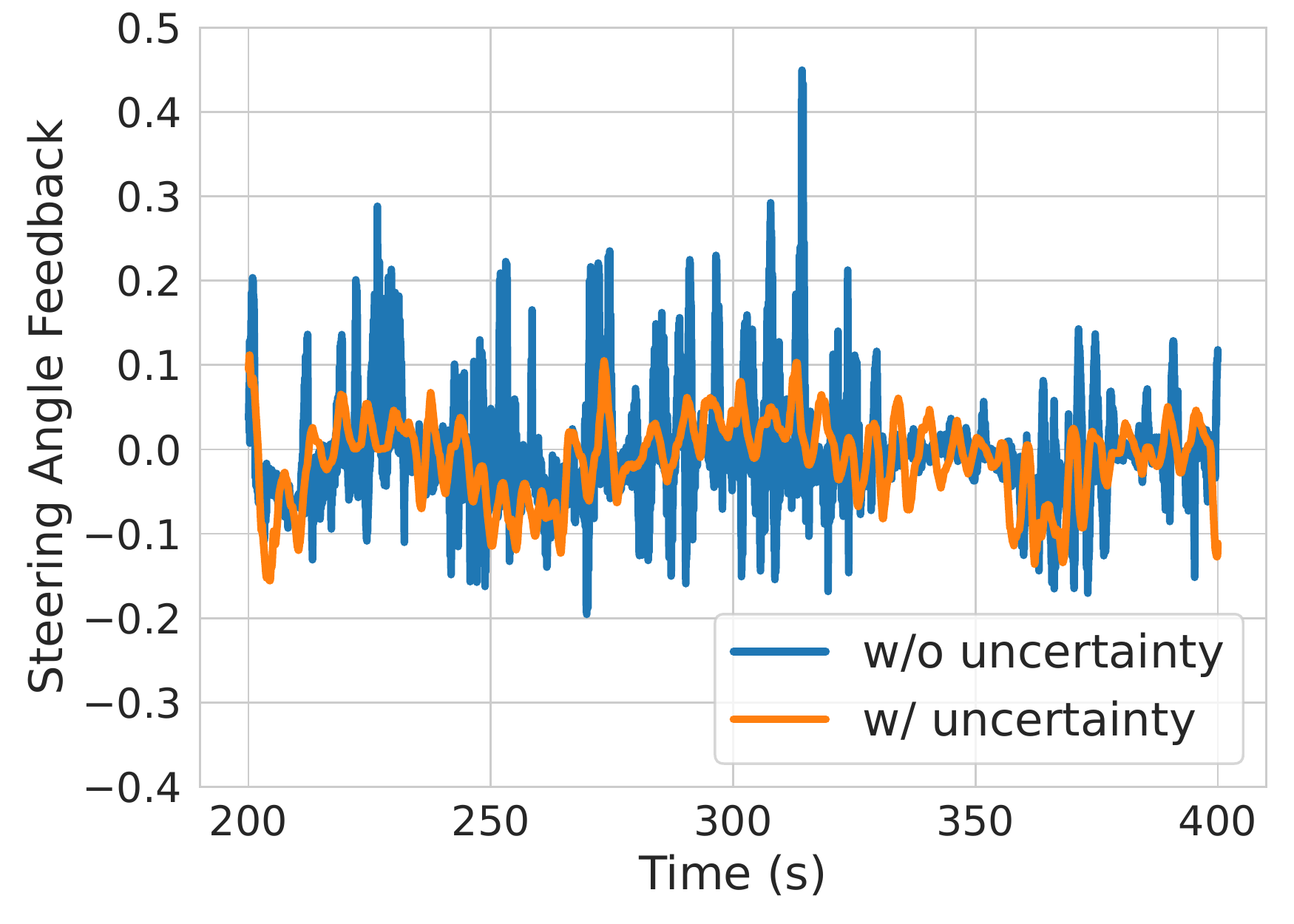}}
	\vspace{-5mm}
	\caption{\textbf{(a)} Real car testing results when the visual input is corrupted by strong sunlight. \textbf{(b)} Steering angle feedback of att-CLF without and with uncertainty propagation on the real car.}
	\label{fig:traj_facing_sun}
	\vspace{-5mm}
\end{figure}
We deploy our end-to-end policy on a full-scale autonomous vehicle. 
We test the policies for 3 trials at different times of the day. Human make intervention when the vehicle is off the road. \tabref{tab:real_car_quant} shows the proposed att-CLF achieves less human interventions and smaller deviation from the lane center. \figref{fig:devens_loop} overlays the driving trajectories from all trials. We observe that both vanilla E2E and att-CLF fail more frequently at the lower-left sharp turn since direct sunlight to the camera along with the reflection from the snow at the side of the road severely corrupt the RGB image. Also, most interventions occur in the trials around 3-4 pm since the sun at that time most likely points directly to the camera, where lighting conditions and shadows in the real world pose great challenges. 


\noindent\textbf{Unstable maneuver in challenging scenarios.}
In \figref{fig:traj_facing_sun}a, we show an example \footnote{\scriptsize Video: \url{https://drive.google.com/file/d/12g0FctJkV7-F_YnCJD7Eurr8d4S7bhGm/view?usp=sharing}} where the camera is pointed directly to the sun. Under the corruption of sunlight, policies tend to exhibit a chattering maneuver as opposed to running the test with mild lighting conditions (in the evening or at noon). The att-CLF augments end-to-end policies with stability and produces smoother trajectories, while vanilla E2E easily makes the vehicle overshoot around the lane center, shown as the red curves.

\noindent\textbf{Uncertainty.} 
\figref{fig:traj_facing_sun}b shows steering feedback from IMU. Introducing uncertainty greatly improve the smoothness of the control. Furthermore, we observe better stability with uncertainty at trials mostly corrupted by sunlight: mean deviation $1.29$m (w/o uncertainty) vs $0.84$m (w/ uncertainty). 

\vspace{-1.5mm}



\section{Related Work}
\label{sec:related_work}

\noindent\textbf{Lyapunov functions and learning.}
In spite of the huge success, the stability properties of deep learning models are less understood \citep{lechner2020gershgorin}. 
One way to achieve this for neural controllers is Lyapunov stability \citep{slotine1991applied}. The premise is to find a Lyapunov function that synthesizes a Lyapunov-stable controller. For example, \citep{jarvis2003some,majumdar2013control} demonstrated that Lyapunov functions can be obtained by solving a sum-of-squares program for polynomial dynamics. \citep{ravanbakhsh2019learning} proposed a learner-falsifier framework to search for valid Lyapunov functions. Recent works further leveraged the expressive power of neural networks in a Lyapunov framework \citep{richards2018lyapunov,chang2019neural}. 
Other works aim to verify the Lyapunov conditions in neural networks \citep{abate2020formal,chang2019neural,dai2020counter,dai2021lyapunov}.
In this work, we work on control Lyapunov functions (CLFs) \citep{primbs1999nonlinear} and leverage differentiable optimization to construct attention of CLFs.


\noindent\textbf{Differentiable optimization for control.}
Recent advances in differentiable optimization enables seamless integration to neural network \citep{Amos2017,agrawal2019differentiable}. Such compatibility opens new ways of combining learning-based models with classical control, which can be mostly formulated as optimization problems, such as model predictive control \citep{Amos2018}, convex optimal control \citep{okada2017path}, path integral control \citep{agrawal2020learning}, and control barrier function \citep{Xiao2021bnet}. 
In this work, we leverage differentiable QP \citep{Amos2017} for CLFs to augment end-to-end driving policies with stability properties.


\noindent\textbf{End-to-end driving policy learning.} 
Apart from traditional autonomous driving pipelines with independent perception, planning, and control modules, many works \citep{Pomerleau1989,lechner2020neural} have demonstrated the ability of end-to-end policy learning that directly infers control from perception. This framework has been shown to solve various tasks including lane keeping \citep{xu2017end}, obstacle avoidance \citep{wang2021learning,bohez2017sensor}, and navigation \citep{amini2019variational,codevilla2018end}. 
The major pitfall of end-to-end driving policy is the black-box model
that hinders the possibility of introducing properties like stability as opposed to the classical control module. In this work, we leverage recent advances in data-driven simulation \citep{amini2021vista} for sim-to-real transfer and embed stability in end-to-end driving policies.

\vspace{-1.5mm}
\section{Conclusion}
\label{sec:conclusion}

We propose to leverage CLFs to embed vision-based end-to-end driving policies with stability properties and introduce the novel notion of stability attention in CLFs that allows to handle environmental changes and brings flexible learning capability. 
We compare with other baselines and demonstrate the effectiveness of att-CLFs in a high-fidelity simulator and on a real autonomous car.

\acks{
The research was supported in parts by Capgemini. It was also partially sponsored by the United States Air Force Research Laboratory and the United States Air Force Artificial Intelligence Accelerator and was accomplished under Cooperative Agreement Number FA8750-19-2-1000. The views and conclusions contained in this document are those of the authors and should not be interpreted as representing the official policies, either expressed or implied, of the United States Air Force or the U.S. Government. The U.S. Government is authorized to reproduce and distribute reprints for Government purposes notwithstanding any copyright notation herein. This work was further supported by The Boeing Company and the Office of Naval Research (ONR) Grant N00014-18-1-2830.
}


\bibliography{bibtex/bib/myref}

\begin{thebibliography}{36}
\providecommand{\natexlab}[1]{#1}
\providecommand{\url}[1]{\texttt{#1}}
\expandafter\ifx\csname urlstyle\endcsname\relax
  \providecommand{\doi}[1]{doi: #1}\else
  \providecommand{\doi}{doi: \begingroup \urlstyle{rm}\Url}\fi

\bibitem[Abate et~al.(2020)Abate, Ahmed, Giacobbe, and
  Peruffo]{abate2020formal}
Alessandro Abate, Daniele Ahmed, Mirco Giacobbe, and Andrea Peruffo.
\newblock Formal synthesis of lyapunov neural networks.
\newblock \emph{IEEE Control Systems Letters}, 5\penalty0 (3):\penalty0
  773--778, 2020.

\bibitem[Agrawal et~al.(2019)Agrawal, Amos, Barratt, Boyd, Diamond, and
  Kolter]{agrawal2019differentiable}
Akshay Agrawal, Brandon Amos, Shane Barratt, Stephen Boyd, Steven Diamond, and
  J~Zico Kolter.
\newblock Differentiable convex optimization layers.
\newblock \emph{Advances in neural information processing systems}, 32, 2019.

\bibitem[Agrawal et~al.(2020)Agrawal, Barratt, Boyd, and
  Stellato]{agrawal2020learning}
Akshay Agrawal, Shane Barratt, Stephen Boyd, and Bartolomeo Stellato.
\newblock Learning convex optimization control policies.
\newblock In \emph{Learning for Dynamics and Control}, pages 361--373. PMLR,
  2020.

\bibitem[Ames et~al.(2012)Ames, Galloway, and Grizzle]{Aaron2012}
A.~D. Ames, K.~Galloway, and J.~W. Grizzle.
\newblock Control lyapunov functions and hybrid zero dynamics.
\newblock In \emph{Proc. of 51rd IEEE Conference on Decision and Control},
  pages 6837--6842, 2012.

\bibitem[Amini et~al.(2019)Amini, Rosman, Karaman, and
  Rus]{amini2019variational}
Alexander Amini, Guy Rosman, Sertac Karaman, and Daniela Rus.
\newblock Variational end-to-end navigation and localization.
\newblock In \emph{2019 International Conference on Robotics and Automation
  (ICRA)}, pages 8958--8964. IEEE, 2019.

\bibitem[Amini et~al.(2022)Amini, Wang, Gilitschenski, Schwarting, Liu, Han,
  Karaman, and Rus]{amini2021vista}
Alexander Amini, Tsun-Hsuan Wang, Igor Gilitschenski, Wilko Schwarting, Zhijian
  Liu, Song Han, Sertac Karaman, and Daniela Rus.
\newblock Vista 2.0: An open, data-driven simulator for multimodal sensing and
  policy learning for autonomous vehicles.
\newblock In \emph{2022 International Conference on Robotics and Automation
  (ICRA)}. IEEE, 2022.

\bibitem[Amos and Kolter(2017)]{Amos2017}
Brandon Amos and J.~Zico Kolter.
\newblock Optnet: Differentiable optimization as a layer in neural networks.
\newblock In \emph{Proceedings of the 34th International Conference on Machine
  Learning - Volume 70}, pages 136--145, 2017.

\bibitem[Amos et~al.(2018)Amos, Rodriguez, Sacks, Boots, and Kolter]{Amos2018}
Brandon Amos, Ivan Dario~Jimenez Rodriguez, Jacob Sacks, Byron Boots, and
  J.~Zico Kolter.
\newblock Differentiable mpc for end-to-end planning and control.
\newblock In \emph{Proceedings of the 32nd International Conference on Neural
  Information Processing Systems}, page 8299–8310. Curran Associates Inc.,
  2018.

\bibitem[Artstein(1983)]{Artstein1983}
Z.~Artstein.
\newblock Stabilization with relaxed controls.
\newblock \emph{Nonlinear Analysis: Theory, Methods \& Applications},
  7\penalty0 (11):\penalty0 1163--1173, 1983.

\bibitem[Bohez et~al.(2017)Bohez, Verbelen, De~Coninck, Vankeirsbilck, Simoens,
  and Dhoedt]{bohez2017sensor}
Steven Bohez, Tim Verbelen, Elias De~Coninck, Bert Vankeirsbilck, Pieter
  Simoens, and Bart Dhoedt.
\newblock {Sensor Fusion for Robot Control through Deep Reinforcement
  Learning}.
\newblock In \emph{IEEE/RSJ International Conference on Intelligent Robots and
  Systems (IROS)}, 2017.

\bibitem[Chang et~al.(2019)Chang, Roohi, and Gao]{chang2019neural}
Ya-Chien Chang, Nima Roohi, and Sicun Gao.
\newblock Neural lyapunov control.
\newblock \emph{Advances in neural information processing systems}, 32, 2019.

\bibitem[Codevilla et~al.(2018)Codevilla, Miiller, L{\'o}pez, Koltun, and
  Dosovitskiy]{codevilla2018end}
Felipe Codevilla, Matthias Miiller, Antonio L{\'o}pez, Vladlen Koltun, and
  Alexey Dosovitskiy.
\newblock {End-to-End Driving via Conditional Imitation Learning}.
\newblock In \emph{IEEE International Conference on Robotics and Automation
  (ICRA)}, 2018.

\bibitem[Dai et~al.(2020)Dai, Landry, Pavone, and Tedrake]{dai2020counter}
Hongkai Dai, Benoit Landry, Marco Pavone, and Russ Tedrake.
\newblock Counter-example guided synthesis of neural network lyapunov functions
  for piecewise linear systems.
\newblock In \emph{2020 59th IEEE Conference on Decision and Control (CDC)},
  pages 1274--1281. IEEE, 2020.

\bibitem[Dai et~al.(2021)Dai, Landry, Yang, Pavone, and
  Tedrake]{dai2021lyapunov}
Hongkai Dai, Benoit Landry, Lujie Yang, Marco Pavone, and Russ Tedrake.
\newblock Lyapunov-stable neural-network control.
\newblock \emph{arXiv preprint arXiv:2109.14152}, 2021.

\bibitem[Freeman and Kokotovic(1996)]{Freeman1996}
R.~A. Freeman and P.~V. Kokotovic.
\newblock \emph{Robust Nonlinear Control Design}.
\newblock Birkhauser, 1996.

\bibitem[Gal and Ghahramani(2016)]{gal2016dropout}
Yarin Gal and Zoubin Ghahramani.
\newblock Dropout as a bayesian approximation: Representing model uncertainty
  in deep learning.
\newblock In \emph{international conference on machine learning}, pages
  1050--1059. PMLR, 2016.

\bibitem[Jarvis-Wloszek et~al.(2003)Jarvis-Wloszek, Feeley, Tan, Sun, and
  Packard]{jarvis2003some}
Zachary Jarvis-Wloszek, Ryan Feeley, Weehong Tan, Kunpeng Sun, and Andrew
  Packard.
\newblock Some controls applications of sum of squares programming.
\newblock In \emph{42nd IEEE international conference on decision and control
  (IEEE Cat. No. 03CH37475)}, volume~5, pages 4676--4681. IEEE, 2003.

\bibitem[Kendall and Gal(2017)]{kendall2017uncertainties}
Alex Kendall and Yarin Gal.
\newblock What uncertainties do we need in bayesian deep learning for computer
  vision?
\newblock \emph{Advances in neural information processing systems}, 30, 2017.

\bibitem[Khalil(2002)]{Khalil2002}
Hassan~K. Khalil.
\newblock \emph{Nonlinear Systems}.
\newblock Prentice Hall, third edition, 2002.

\bibitem[Lakshminarayanan et~al.(2017)Lakshminarayanan, Pritzel, and
  Blundell]{lakshminarayanan2017simple}
Balaji Lakshminarayanan, Alexander Pritzel, and Charles Blundell.
\newblock Simple and scalable predictive uncertainty estimation using deep
  ensembles.
\newblock \emph{Advances in neural information processing systems}, 30, 2017.

\bibitem[Lechner et~al.(2020{\natexlab{a}})Lechner, Hasani, Amini, Henzinger,
  Rus, and Grosu]{lechner2020neural}
Mathias Lechner, Ramin Hasani, Alexander Amini, Thomas~A Henzinger, Daniela
  Rus, and Radu Grosu.
\newblock Neural circuit policies enabling auditable autonomy.
\newblock \emph{Nature Machine Intelligence}, 2\penalty0 (10):\penalty0
  642--652, 2020{\natexlab{a}}.

\bibitem[Lechner et~al.(2020{\natexlab{b}})Lechner, Hasani, Rus, and
  Grosu]{lechner2020gershgorin}
Mathias Lechner, Ramin Hasani, Daniela Rus, and Radu Grosu.
\newblock Gershgorin loss stabilizes the recurrent neural network compartment
  of an end-to-end robot learning scheme.
\newblock In \emph{2020 IEEE International Conference on Robotics and
  Automation (ICRA)}, pages 5446--5452. IEEE, 2020{\natexlab{b}}.

\bibitem[Majumdar et~al.(2013)Majumdar, Ahmadi, and
  Tedrake]{majumdar2013control}
Anirudha Majumdar, Amir~Ali Ahmadi, and Russ Tedrake.
\newblock Control design along trajectories with sums of squares programming.
\newblock In \emph{2013 IEEE International Conference on Robotics and
  Automation}, pages 4054--4061. IEEE, 2013.

\bibitem[Okada et~al.(2017)Okada, Rigazio, and Aoshima]{okada2017path}
Masashi Okada, Luca Rigazio, and Takenobu Aoshima.
\newblock Path integral networks: End-to-end differentiable optimal control.
\newblock \emph{arXiv preprint arXiv:1706.09597}, 2017.

\bibitem[Pomerleau(1989)]{Pomerleau1989}
Dean~A Pomerleau.
\newblock {{ALVINN}}: {{An}} autonomous land vehicle in a neural network.
\newblock In \emph{Advances in Neural Information Processing Systems
  ({{NeurIPS}})}, 1989.

\bibitem[Primbs et~al.(1999)Primbs, Nevisti{\'c}, and
  Doyle]{primbs1999nonlinear}
James~A Primbs, Vesna Nevisti{\'c}, and John~C Doyle.
\newblock Nonlinear optimal control: A control lyapunov function and receding
  horizon perspective.
\newblock \emph{Asian Journal of Control}, 1\penalty0 (1):\penalty0 14--24,
  1999.

\bibitem[Ravanbakhsh and Sankaranarayanan(2019)]{ravanbakhsh2019learning}
Hadi Ravanbakhsh and Sriram Sankaranarayanan.
\newblock Learning control lyapunov functions from counterexamples and
  demonstrations.
\newblock \emph{Autonomous Robots}, 43\penalty0 (2):\penalty0 275--307, 2019.

\bibitem[Richards et~al.(2018)Richards, Berkenkamp, and
  Krause]{richards2018lyapunov}
Spencer~M Richards, Felix Berkenkamp, and Andreas Krause.
\newblock The lyapunov neural network: Adaptive stability certification for
  safe learning of dynamical systems.
\newblock In \emph{Conference on Robot Learning}, pages 466--476. PMLR, 2018.

\bibitem[Rucco et~al.(2015)Rucco, Notarstefano, and Hauser]{Rucco2015}
Alessandro Rucco, Giuseppe Notarstefano, and John Hauser.
\newblock An efficient minimum-time trajectory generation strategy for
  two-track car vehicles.
\newblock \emph{IEEE Transactions on Control Systems Technology}, 23\penalty0
  (4):\penalty0 1505--1519, 2015.

\bibitem[Scott(2015)]{scott2015multivariate}
David~W Scott.
\newblock \emph{Multivariate density estimation: theory, practice, and
  visualization}.
\newblock John Wiley \& Sons, 2015.

\bibitem[Slotine et~al.(1991)Slotine, Li, et~al.]{slotine1991applied}
Jean-Jacques~E Slotine, Weiping Li, et~al.
\newblock \emph{Applied nonlinear control}, volume 199.
\newblock Prentice hall Englewood Cliffs, NJ, 1991.

\bibitem[Sontag(1983)]{Sontag1983}
E.~Sontag.
\newblock A lyapunov-like stabilization of asymptotic controllability.
\newblock \emph{SIAM Journal of Control and Optimization}, 21\penalty0
  (3):\penalty0 462--471, 1983.

\bibitem[Wang et~al.(2021)Wang, Amini, Schwarting, Gilitschenski, Karaman, and
  Rus]{wang2021learning}
Tsun-Hsuan Wang, Alexander Amini, Wilko Schwarting, Igor Gilitschenski, Sertac
  Karaman, and Daniela Rus.
\newblock Learning interactive driving policies via data-driven simulation.
\newblock \emph{arXiv preprint arXiv:2111.12137}, 2021.

\bibitem[Xiao and Belta(2019)]{Xiao2019}
Wei Xiao and Calin Belta.
\newblock Control barrier functions for systems with high relative degree.
\newblock In \emph{Proc. of 58th IEEE Conference on Decision and Control},
  pages 474--479, Nice, France, 2019.

\bibitem[Xiao et~al.(2023)Xiao, Wang, Hasani, Chahine, Amini, Li, , and
  Rus]{Xiao2021bnet}
Wei Xiao, Tsun-Hsuan Wang, Ramin Hasani, Makram Chahine, Alexander Amini, Xiao
  Li, , and Daniela Rus.
\newblock Barriernet: Differentiable control barrier functions for learning of
  safe robot control.
\newblock \emph{IEEE Transactions on Robotics}, 2023.

\bibitem[Xu et~al.(2017)Xu, Gao, Yu, and Darrell]{xu2017end}
Huazhe Xu, Yang Gao, Fisher Yu, and Trevor Darrell.
\newblock {End-to-End Learning of Driving Models from Large-Scale Video
  Datasets}.
\newblock In \emph{IEEE Conference on Computer Vision and Pattern Recognition
  (CVPR)}, 2017.

\end{thebibliography}

\end{document}